\documentclass[conference]{IEEEtran}
\IEEEoverridecommandlockouts
\usepackage{cite}
\usepackage{amsmath,amssymb,amsfonts}
\usepackage[algo2e]{algorithm2e}
\usepackage{algorithm}
\usepackage{algorithmic}
\usepackage{graphicx}
\usepackage{textcomp}
\usepackage{xcolor}
\usepackage{mathtools}
\usepackage{amsthm}
\usepackage{adjustbox}
\usepackage{multirow}
\usepackage{threeparttable}
\usepackage{url}

\theoremstyle{definition}

\newtheorem{theorem}{Theorem}[section]

\usepackage[a4paper, total={184mm,239mm}]{geometry}
\def\BibTeX{{\rm B\kern-.05em{\sc i\kern-.025em b}\kern-.08em
    T\kern-.1667em\lower.7ex\hbox{E}\kern-.125emX}}
    
\begin{document}
\title{IterL2Norm: Fast Iterative L2-Normalization\\
\thanks{This research was supported by National R\&D Program through the National Research Foundation of Korea (NRF) and Institute of Information \& communications
Technology Planning \& Evaluation (IITP) grants funded by the Korea
government (MSIT) (RS-2024-00406897 and IITP-(2024)-RS-2023-00253914). The EDA tool was supported by the IC Design Education Center (IDEC), Korea.}}

\author{ChangMin Ye$^{\dag}$, Yonguk Sim$^{\ddag}$, Youngchae Kim$^{\ddag}$, SeongMin Jin$^{\dag}$, and Doo Seok Jeong$^{\dag\ddag*}$\\
$^{\dag}$Division of Materials Science and Engineering, Hanyang University, Seoul, Republic of Korea\\
$^\ddag$Department of Semiconductor Engineering, Hanyang University, Seoul, Republic of Korea\\
\IEEEauthorblockA{*Corresponding author: $\text{dooseokj}@\text{hanyang.ac.kr}$
}}


\maketitle

\begin{abstract}
Transformer-based large language models are a memory-bound model whose operation is based on a large amount of data that are marginally reused. Thus, the data movement between a host and accelerator likely dictates the total wall-clock time. Layer normalization is one of the key workloads in the transformer model, following each of multi-head attention and feed-forward network blocks. To reduce data movement, layer normalization needs to be performed on the same chip as the matrix-matrix multiplication engine. To this end, we introduce an iterative L2-normalization method for 1D input (IterL2Norm), ensuring fast convergence to the steady-state solution within five iteration steps and high precision, outperforming the fast inverse square root algorithm in six out of nine cases for FP32 and five out of nine for BFloat16 across the embedding lengths used in the OPT models. Implemented in 32/28nm CMOS, the IterL2Norm macro normalizes $d$-dimensional vectors, where $64 \leq d \leq 1024$, with a latency of 116-227 cycles at 100MHz/1.05V.
\end{abstract}


\section{Introduction}
Large language models (LLMs) such as GPT~\cite{radford2018improving}, Gemini~\cite{team2023gemini}, and Llama~\cite{touvron2023llama} represent a recent breakthrough with profound impacts on global society. These state-of-the-art LLMs commonly adopt the transformer architecture~\cite{vaswani2017attention} that ensures high performance in natural language processing due to self-attention with an explicit working memory. In particular, the decoder-only transformer architecture employed in these LLMs ensures high-performance associative recalls in a generative manner. The decoder-only transformer consists of multiple decoders in series, each of which consists of a masked multi-head attention and feed-forward network sub-block in series. Matrix-matrix multiplication (MatMul) operations in these sub-blocks represent a major workload. Notably, each sub-block is followed by layer normalization that L2-normalizes the output for each batch.

LLMs based on transformer are memory-bound models that depend on a large amount of data but very limited reuse of them for operations (cf. convolutional neural networks as representative compute-bound models)~\cite{kim2023full}. The large amount of data needed and their limited reuse cause significant data traffic between a main memory (DRAM) and processor, so that the overall operational wall-clock time is dictated by the memory bandwidth rather than the processor performance. Thus, graphics processing units (GPUs) equipped with high bandwidth memories are widely used to accelerate LLMs. As an alternative to GPUs, various processors have recently been proposed to mainly accelerate MatMul operations and/or activation functions at a lower power than GPUs, including computing-in-memory units such as Function-In-Memory DRAM (FIMDRAM)~\cite{kwon202125}, and GDDR6-based accelerator-in-memory (AiM)~\cite{lee20221ynm}. However, most of them rely on their host for layer normalization. That is, the output from each sub-block should be sent to the main memory for layer normalization, which leads to data traffic, and thus significant latency and power consumption~\cite{kim2023full}.

For layer normalization to be performed on the same chip as MatMul processing engines, additional arithmetic operations such as square root and division need to be performed. However, the area/power overheads are probably prohibitive, as for~\cite{marchisio2023swifttron}. Alternatively, various approximations without division operations have been implemented using digital logic circuits such as~\cite{yu2022nn,wu2024pim,wang2023sole}. Unfortunately, detailed implementations and performance data are not well documented.

The layer normalization algorithm for on-chip implementation needs to avoid vanilla division and square root and to be generic to apply to various floating-point (FP) formats with low power/area overheads and operational latency. To this end, we propose a fast iterative L2-normalization algorithm (IterL2Norm) that is based on a high-dimensional dynamic system with a few fixed point. One of them represents the L2-normalized 1D vector in the hyper-space, which can be attained by appropriately setting the initial point in the hyper-space. IterL2Norm is a division and square root operation-free L2-normalization algorithm, rendering it suitable for power- and area-efficient on-chip implementations. It is based on a sold theoretical ground and applicable to various FP formats unlike previous methods tailored to specific FP formats~\cite{fast_inverse_square_root}. It highlights fast convergence toward the fixed point (L2-normalized vector) in the hyperspace within five iteration steps, and thus low operational latency.


The primary contributions of our work include
\begin{itemize}
    \item We introduce a novel L2-normalization algorithm on a solid theoretical (rather than heuristic) ground with a full derivation.
    \item We share our results of in-depth analyses on (i) its precision, convergence rate, and latency for various data lengths in FP32/FP16/BFloat16 and (ii) LLM-level performance by replacing the conventional layer normalization algorithm.  
    \item We also share our digital implementation of the IterL2Norm macro with detailed explanation and its power and area overheads in 32/28nm CMOS technology. 
\end{itemize}
The rest of this paper is organized as follows. Sec.~\ref{sec:preliminaries} explains the key theories of dynamic systems that prototype IterL2Norm. Sec.~\ref{sec:methods} elaborates IterL2Norm and the system initialization that ensures fast convergence. Sec.~\ref{sec:macro_design} proposes the IterL2Norm macro architecture. Sec.~\ref{sec:eval} evaluates IterL2Norm and the macro, and Sec.~\ref{sec:comp} compares our work with previous ones. Sec.~\ref{sec:conclusion} concludes our work.

\section{Preliminaries}
\label{sec:preliminaries}
\begin{theorem}\label{theorem1}
Let $\boldsymbol{y}$ and $\tilde{\boldsymbol{y}}$ be vectors of the same length. Let $k$ be a nonzero scalar value such that $k=\boldsymbol{y}\cdot\tilde{\boldsymbol{y}}$. Consider the following differential equation for $\tilde{\boldsymbol{y}}$ for a given $\boldsymbol{y}$.  
\begin{equation}\label{equ:govern1}
\tau\dfrac{d\tilde{\boldsymbol{y}}}{dt}=k\boldsymbol{y} - \alpha k^2\tilde{\boldsymbol{y}}\textrm{,}
\end{equation}
where $\alpha$ is a positive constant. For a given $\boldsymbol{y}$, $\tilde{\boldsymbol{y}}$ is initialized to $\tilde{\boldsymbol{y}}_0$ such that $k=\boldsymbol{y}\cdot\tilde{\boldsymbol{y}}_0>0$. The steady state solution to this differential equation ($\tilde{\boldsymbol{y}}_\infty$) satisfies that $\left\|\tilde{\boldsymbol{y}}_\infty\right\|^2_2 = \alpha^{-1}$ and $\tilde{\boldsymbol{y}}_\infty = \alpha^{-1/2}\boldsymbol{y}/\left\|\boldsymbol{y}\right\|_2$.  
\end{theorem}
\begin{proof}
Given that $k=\boldsymbol{y}\cdot\tilde{\boldsymbol{y}}$, the inner product of each side of Eq.~\eqref{equ:govern1} and $\boldsymbol{y}$ yields
\begin{equation*}
\tau\left(\dfrac{dk}{dt}-\tilde{\boldsymbol{y}}\cdot \dfrac{d\boldsymbol{y}}{dt}\right)=\tau\dfrac{dk}{dt}=k\left\|\boldsymbol{y}\right\|_2^2 - \alpha k^3\textrm{.}
\end{equation*}
For a positive $\alpha$, this dynamic system holds one unstable fixed point ($k = 0$) and two stable fixed points ($k=\pm\alpha^{-1/2}\left\|\boldsymbol{y}\right\|_2$). Therefore, the steady state $k$ ($k_\infty$) is determined by the initial $k$ ($k_0=\boldsymbol{y}\cdot\tilde{\boldsymbol{y}}_0$) such that $k_\infty = \alpha^{-1/2}\left\|\boldsymbol{y}\right\|_2$ if $k_0>0$ and $k_\infty = -\alpha^{-1/2}\left\|\boldsymbol{y}\right\|_2$ if $k_0<0$. Because we consider only $\tilde{\boldsymbol{y}}_0$ that leads to $k_0>0$, we have 
\begin{equation}\label{equ:ss_k}
k_\infty = \alpha^{-1/2}\left\|\boldsymbol{y}\right\|_2\textrm{.}
\end{equation}

The inner product of each side of Eq.~\eqref{equ:govern1} and $\tilde{\boldsymbol{y}}$ yields
\begin{equation*}
\dfrac{\tau}{2}\dfrac{d\left\|\tilde{\boldsymbol{y}}\right\|_2^2}{dt}=k^2 - \alpha k^2\left\|\tilde{\boldsymbol{y}}\right\|_2^2\textrm{.}
\end{equation*}
In the steady state, the left-hand side is zero, so that we have  $\left\|\tilde{\boldsymbol{y}}_\infty\right\|_2^2=\alpha^{-1}$. Additionally, the steady state solution to Eq.~\eqref{equ:govern1} is $\tilde{\boldsymbol{y}}_\infty = \boldsymbol{y}/\alpha k_\infty$. Using Eq.~\eqref{equ:ss_k}, we eventually have $\tilde{\boldsymbol{y}}_\infty = \alpha^{-1/2}\boldsymbol{y}/\left\|\boldsymbol{y}\right\|_2$.
\end{proof} 

\section{Proposed method}\label{sec:methods}
\subsection{Iterative normalization method}
Layer normalization for a given vector $\boldsymbol{x}\in\mathbb{R}^d$ involves the following sequential steps:\\
\textbf{Step 1:} Shifting the mean of $\boldsymbol{x}$ ($\overline{\boldsymbol{x}}$) to zero, $\boldsymbol{y}\leftarrow\boldsymbol{x}-\overline{\boldsymbol{x}}$,\\
\textbf{Step 2:} Normalizing $\boldsymbol{y}$ by the standard deviation of $\boldsymbol{y}$, i.e., $\sigma_y$, $\hat{\boldsymbol{y}}\leftarrow\boldsymbol{y}/\sigma_y$,\\
\textbf{Step 3:} Scaling and shifting $\hat{\boldsymbol{y}}$, $\boldsymbol{z}\leftarrow\boldsymbol{\gamma}\cdot\hat{\boldsymbol{y}}+\boldsymbol{\beta}$.\\
Because $\sigma_y=d^{-1/2}\left\|\boldsymbol{y}\right\|_2$, \textbf{Step 2} is expressed as $\hat{\boldsymbol{y}}\leftarrow d^{1/2}\boldsymbol{y}/\left\|\boldsymbol{y}\right\|_2$, which is equivalent to L2-normalizing $\boldsymbol{y}$ and subsequently scaling it by multiplying $d^{1/2}$. \textbf{Step 2} is the only step involving division operations which are computationally expensive. We can replace this costly step by the iterative method supported by \textbf{Theorem}~\ref{theorem1}. \textbf{Theorem}~\ref{theorem1} for $\alpha=1$ explains that the following differential equation,
\begin{equation}\label{equ:govern2}
\tau\dfrac{d\tilde{\boldsymbol{y}}}{dt}=k\boldsymbol{y} - k^2\tilde{\boldsymbol{y}}\textrm{, where }k=\boldsymbol{y}\cdot\tilde{\boldsymbol{y}}\textrm{,}
\end{equation}
has the steady state solution $\tilde{\boldsymbol{y}}_\infty=\boldsymbol{y}/\left\|\boldsymbol{y}\right\|_2$. Thus, we can evaluate $\tilde{\boldsymbol{y}}$ for $\boldsymbol{y}$ by solving Eq.~\eqref{equ:govern2}. We solve Eq.~\eqref{equ:govern2} by approximating it to a recursive form using the Euler method as follows.
\begin{equation}\label{equ:approx1}
\tilde{\boldsymbol{y}}_{i+1} = \left(1-\lambda  k_i^2\right)\tilde{\boldsymbol{y}}_{i} + \lambda k_i\boldsymbol{y}\textrm{,}
\end{equation}
where $\lambda=\Delta t/\tau$, and $k_i=\boldsymbol{y}\cdot\tilde{\boldsymbol{y}}_i$. Note that the subscript $i$ for $\tilde{\boldsymbol{y}}$ and $k$ denotes the $i$th iteration step. The timestep width is denoted by $\Delta t$. The value $\tilde{\boldsymbol{y}}_{i+1}$ in Eq.~\eqref{equ:approx1} is repeatedly calculated until it reaches its steady state, yielding $\tilde{\boldsymbol{y}}_\infty$. For all $i$, $\tilde{\boldsymbol{y}}_{i}$ is parallel to $\boldsymbol{y}$, so that we replace $\tilde{\boldsymbol{y}}_{i}$ by $a_i\boldsymbol{y}$, leading to $k_i=a_i\left\|\boldsymbol{y}\right\|_2^2$ Thus, Eq.~\eqref{equ:approx1} is converted into a simple scalar equation.
\begin{equation}\label{equ:approx2}
\Delta a = a_{i+1} - a_i = \lambda\left\|\boldsymbol{y}\right\|_2^2 a_i\left(1-\left\|\boldsymbol{y}\right\|_2^2a_i^2\right)\textrm{,}
\end{equation}
which asymptotically converges towards $a_\infty=1/\left\|\boldsymbol{y}\right\|_2$ with a positive $a_0$ and sufficiently small $\lambda$. We refer to this L2-normalization method with the replacement of \textbf{Step 2} by this iterative normalization as IterL2Norm. The pseudocode for IterL2Norm-based layer normalization is shown in \textbf{Algorithm}~\ref{algo:ibn}.

\begin{algorithm}[tb!]
\normalsize
\SetAlgoLined
\SetAlCapNameFnt{\normalsize}
\SetAlCapFnt{\normalsize}
\textbf{Input:} input vector $\boldsymbol{x}$; update-rate $\lambda$; max-tolerated error $\delta_\textrm{max}$; scale and shift parameters ($\boldsymbol{\gamma}$ and $\boldsymbol{\beta}$)\\
\textbf{Output:} output vector $\boldsymbol{z}$\\
\textbf{Initialization: } $\Delta a\gets \delta_0 \left(>\delta_\textrm{max}\right)$; $a\gets a_0$\\
$\overline{x}\gets d^{-1}\sum_{i=1}^dx_i$\\
$\boldsymbol{y}\gets \boldsymbol{x} - \overline{x}$\\
$m\gets \left\|\boldsymbol{y}\right\|_2^2$\\
\While{$\Delta a>\delta_\textrm{max}$}{
$\Delta a\gets \lambda ma \left(1-ma^2\right)$\\
$a\gets a + \Delta a$}
$\hat{\boldsymbol{y}}\gets d^{1/2}a\boldsymbol{y}$\\
$\boldsymbol{z}\gets \boldsymbol{\gamma}\hat{\boldsymbol{y}}+\boldsymbol{\beta}$
\caption{IterL2Norm-based layer normalization.}
\label{algo:ibn}
\end{algorithm}

\begin{figure*}[tbh!]
\centerline{\includegraphics[width=0.95\linewidth]{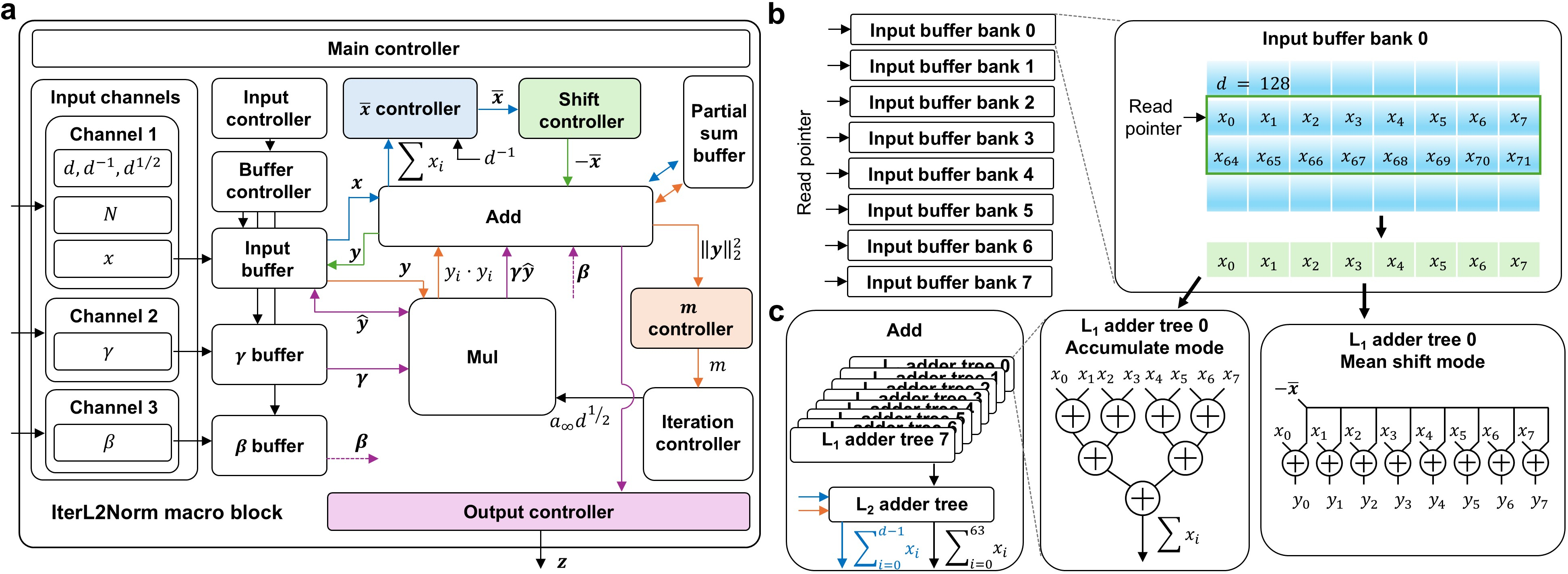}}
\caption{(\textbf{a}) Architecture of the IterL2Norm macro. (\textbf{b}) Data organization in the Input buffer. (\textbf{c}) Block diagram of the Add block equipped with total nine 8-input adder trees.}\label{fig:Macro_overview}
\end{figure*}

\subsection{Initialization and update-rate setting}
The more the iteration steps for IterL2Norm, the larger the wall-clock time. To shorten the iteration, we should use the initial value $a_0$ close to $a_{\infty}$$(=1/\left\|\boldsymbol{y}\right\|_2)$
and update-rate $\lambda$ such that (i) $\lambda$ is sufficiently large for $a_\infty$ to be attained with the minimum iteration steps (fast convergence) but (ii) sufficiently small to avoid an intolerable error in approximation.\\
\textbf{Initialization of $\boldsymbol{a}$: } Let $m$ be $\left\|\boldsymbol{y}\right\|_2^2$, which is once evaluated for IterL2Norm as shown in \textbf{Algorithm}~\ref{algo:ibn}. We initialize $a$ using the exponent of $m$, $E\left(m\right)$, as follows.
\begin{equation}\label{equ:a_init}
a_0 = 2^{-\left(E\left(m\right)-\text{bias}+1\right)/2}\textrm{,}
\end{equation}
where bias depends on the data format, e.g., $\text{bias}=127$ for FP32 and BFloat16 and $\text{bias}=15$ for FP16. Because $a_{\infty}=m^{-1/2}$, we can express $a_{\infty}$ as follows.
\begin{equation*}
a_{\infty} = \text{Significand}\left(m\right)^{-1/2}\cdot2^{-\left(E\left(m\right)-\text{bias}\right)/2}\textrm{,}
\end{equation*}
where Significand$(m)$ denotes the significand of $m$, which satisfies $1\leq\text{Significand}(m)<2$. Therefore, we have $0.7<a_0/a_{\infty}<1$, implying that $a_0$ is already close to $a_{\infty}$ in so much as the distance is smaller than 30\% of $a_{\infty}$. Further, the evaluation of $a_0$ involves one addition, one subtraction, and one bit-shift operation only.\\
\textbf{Update-rate $\lambda$: }Eq.~\eqref{equ:approx2} can be expressed as the following differential equation.
\begin{equation}\label{equ:a_continuous}
\tau\dfrac{da}{dt} = -m^2a\left(a^2-1/m\right)\textrm{,}
\end{equation}
where $m=\left\|\boldsymbol{y}\right\|_2^2$. There exists the analytical solution to Eq.~\eqref{equ:a_continuous}.
\begin{equation}\label{equ:a_sol}
a = a_0\left[\left(1-ma_0^2\right)e^{-2mt/\tau}+ma_0^2\right]^{-1/2}\textrm{.}
\end{equation}
Because we consider discrete iteration steps, we replace $t$ in Eq.~\eqref{equ:a_sol} by $n\Delta t$ with non-negative integer $n$ that indicates the iteration step index. Subsequently, by introducing $\lambda\left(=\Delta t/\tau\right)$, we have
\begin{equation}\label{equ:a_sol2}
a = a_0\left[\left(1-ma_0^2\right)e^{-2mn\lambda}+ma_0^2\right]^{-1/2}\textrm{.}
\end{equation}
The convergence rate is determined by the exponent on the right-hand side of Eq.~\eqref{equ:a_sol2}. 
For fast convergence, the exponential term should fall below a tolerable error value $\delta_c \sim 0$ within a few iteration steps $n_c$, leading to the following inequality, $\lambda > -\left(2mn_c\right)^{-1}\ln\delta_c$. We set $\delta_c$ and $n_c$ to $10^{-3}$ and $5$, respectively, so that we have $\lambda > 0.69m^{-1}$. As such, the calculation of $m^{-1}$ needs a division operation, which is avoided in IterL2Norm. Because the exponent of $m$, i.e., $E\left(m\right)$ is known, the range of $m^{-1}$ is readily available, $0.5\cdot2^{-\left(E\left(m\right)-\text{bias}\right)}<m^{-1}\leq 2^{-\left(E\left(m\right)-\text{bias}\right)}$. Therefore, we approximate the condition of $\lambda$ for $a$ to converge within $n_c\left(=5\right)$ iteration steps.
\begin{equation}\label{equ:rate_init}
\lambda > 0.345\cdot 2^{-\left(E\left(m\right)-\text{bias}\right)}\textrm{.}
\end{equation}
This calculation needs one subtraction and one multiplication operation only. 

\section{IterL2Norm macro design}\label{sec:macro_design}


\begin{figure}[tb]
\centerline{\includegraphics[width=0.95\linewidth]{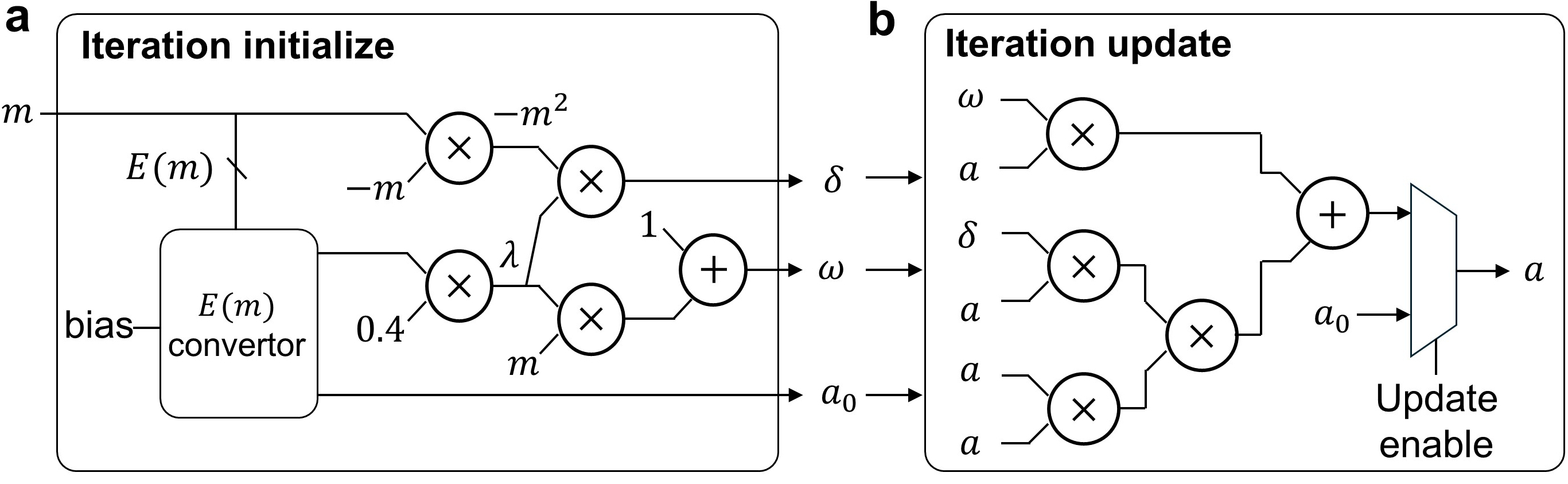}}
\caption{Architecture of (\textbf{a}) the initialize and (\textbf{b}) the update modules in the iteration controller.}\label{fig:iteration_controller}
\end{figure}
The IterL2Norm macro implements the IterL2Norm-based layer normalization algorithm for a $d$-long input vector $\boldsymbol{x}=\left[x_0,x_1,\cdots,x_{d-1}\right]$ with scale parameters $\boldsymbol{\gamma}=\left[\gamma_{0},\gamma_{1},\cdots,\gamma_{d-1}\right]$ and shift parameters $\boldsymbol{\beta}=\left[\beta_0,\beta_1,\cdots,\beta_{d-1}\right]$, which outputs layer-normalized input $\boldsymbol{z} = \left[z_0,z_1,\cdots,z_{d-1}\right]$. Fig.~\ref{fig:Macro_overview}\textbf{a} shows a block diagram of the IterL2Norm macro proposed. The Input buffer of eight parallel banks ($n_b=8$) buffers a $d$-long input vector, and thus the input length $d$ is limited by the buffer size (Fig.~\ref{fig:Macro_overview}\textbf{b}). Because each bank can stores $16\times 8$ input elements ($h_b=16$ and $w_b=8$), the IterL2Norm macro can handle $d=1024$ maximally for a single input, i.e., $d_\text{max}=1024$.
A $d$-long input vector $\boldsymbol{x}$ is buffered over
multiple banks such that, in a bank $b$ out of total eight banks ($n_b=8$), its row $i$ stores $x[w_b(b+n_bi):w_b(b+n_bi+1)-1]$ as illustrated in Fig.~\ref{fig:Macro_overview}\textbf{b}. Because eight parallel banks share a read pointer, $x[n_bw_bi:n_bw_b(i+1)-1]$ is read at a time.
Instead, multiple ($\lfloor d_\text{max}/d\rfloor$) input vectors can be buffered and sequentially normalized. Note that, to maintain $d_\text{max}=1024$ for FP32/16 and BFloat16, the IterL2Norm macro for FP32 uses the Input buffer twice as large as that for FP16 and BFloat16. Additionally, the Mul and Add blocks are tailored to each data format by using format-specific multipliers and adders but with the same latency of two clock cycles.  

This macro normalizes the input vector using the following sequence.\\
\textbf{Initialization: }The macro is initialized with input length $d$ and number of input vectors.\\
\textbf{Data loading: }The Input, $\gamma$, and $\beta$ buffers are loaded with input vector(s), and scale and shift parameters, respectively, through the input channels. This is controlled by the input and main controllers.\\
\textbf{Mean-shift: }The $\overline{\boldsymbol{x}}$ controller retrieves the input vector from the Input buffer to calculate its element-wise sum in the Add block. The Add block is equipped with eight 8-input L$_1$ adder trees and one 8-input L$_2$ adder tree, which can add 64 input elements to yield the sum of the partial input at a time (Fig.~\ref{fig:Macro_overview}\textbf{c}). This sum is buffered in the Partial sum buffer alongside the sum values for the previous partial inputs. This is repeated $\lceil d/64\rceil$ times to collect total $\lceil d/64\rceil$ sum values in the Partial sum buffer. They are sent to the Add block to acquire the sum of the whole input vector.
The sum is subsequently multiplied by $d^{-1}$ (pre-stored in the memory), eventually outputting the mean $\overline{\boldsymbol{x}}$. The Shift controller then shifts the mean of $\boldsymbol{x}$ to zero by subtracting $\overline{\boldsymbol{x}}$ from $\boldsymbol{x}$ and rewrites the mean-shifted vector, $\boldsymbol{y}=\boldsymbol{x} - \overline{\boldsymbol{x}}$, into the Input buffer.\\
\textbf{Inner product of $\boldsymbol{y}$ with itself: }The $m$ controller reads the mean-shifted vector $\boldsymbol{y}$ from the Input buffer and sends it to the Mul block (equipped with 64 multipliers). The resulting vector is sent to the Add block that outputs the inner product of a partial vector of $\boldsymbol{y}$. This result is buffered in the Partial sum buffer. This is repeated $\lceil d/64\rceil$ times to calculate $m=\left\|\boldsymbol{y}\right\|_2^2$.\\
\textbf{Iteration: }The Iteration controller initializes $a_0$ using Eq.~\eqref{equ:a_init} and sets the update rate $\lambda$ using Eq.~\eqref{equ:rate_init} (Fig.~\ref{fig:iteration_controller}\textbf{a}). It then iteratively updates $a$ using Eq.~\eqref{equ:approx2} to attain its steady-state value $a_\infty$ (Fig.~\ref{fig:iteration_controller}\textbf{b}). The number of iteration steps $n_c$ is a programmable variable.\\
\textbf{Output: } The Output controller reads the mean-shifted vector $\boldsymbol{y}$ from the Input buffer and sends it to the Mul block with the product of $a_\infty$ and pre-stored $d^{1/2}$ to obtain the L2-normalization result $\hat{\boldsymbol{y}}$. This vector is re-sent to the Mul block with the scale parameters buffered in the $\gamma$ block and the scaled vector to the Add block with the shift parameters buffered in the $\beta$ block to finally obtain the layer-normalization result $\boldsymbol{z}$ for a given input $\boldsymbol{x}$. 

\section{Evaluation}\label{sec:eval}
\subsection{Computational precision and convergence rate}
\begin{figure}[tb]
\centerline{\includegraphics[width=0.95\linewidth]{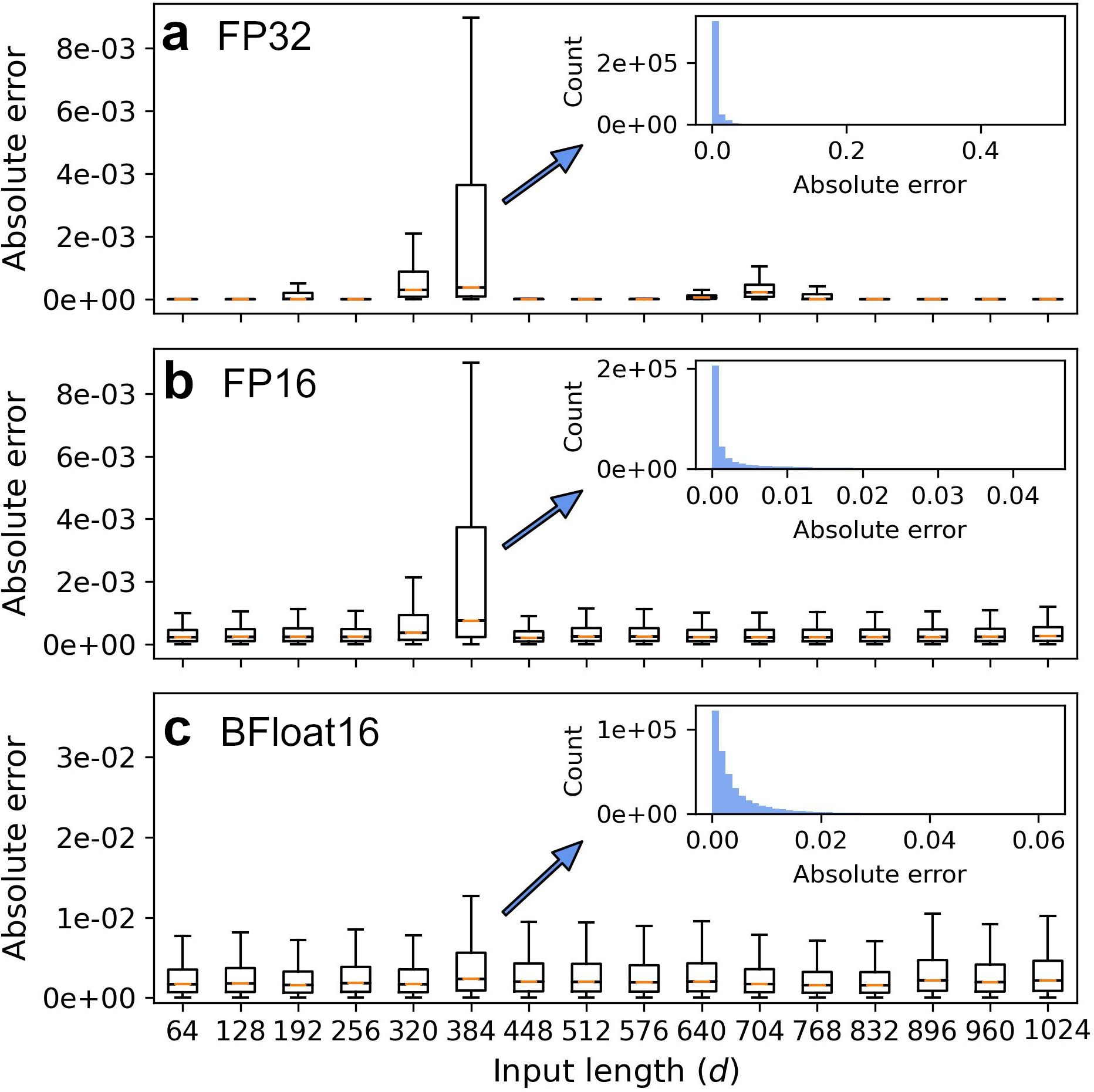}}
\caption{Approximation precision of IterL2Norm for various input lengths $d$ in (\textbf{a}) FP32, (\textbf{b}) FP16, and (\textbf{c}) BFloat16. The insets show the distribution of errors for $d=384$ over 1,000 input vectors.}\label{fig:error}
\end{figure}

We evaluated the performance of the IterL2Norm macro implemented in a Xilinx Virtex-7 FPGA.
We applied the IterL2Norm-based normalization to random vectors of different lengths $(64\leq d\leq 1024)$ in FP32/FP16/BFloat16. For each length and each data format, we used 1,000 random vectors sampled from a uniform distribution in the range $(-1, 1)$ as input vectors. The number of iteration steps was set to 5. We used the absolute deviation of our results from the ground truth (absolute error) as a measure of computational precision. The ground truth values were calculated by applying the layer-normalization function in PyTorch (1.12.1)~\cite{paszke2019pytorch} to the same random vectors using a CPU.
Fig.~\ref{fig:error} shows the evaluation results.
The average (maximum) absolute errors for FP32, FP16, and BFloat16 are $2.23 \times 10^{-4} (5.0\times10^{-1})$, $5.26 \times 10^{-4} (4.9\times10^{-1})$, and $3.07 \times 10^{-3} (6.8\times10^{-1})$, respectively. The maximum error cases marginally occurred as shown in the histograms in the insets of Fig.~\ref{fig:error}. 

LLMs use long embedding vectors, as seen in the OPT models, with the largest (OPT-175B) utilizing 12,288-dimensional embeddings~\cite{zhang2022opt}. We further evaluated the precision of IterL2Norm for the embedding lengths used in the OPT models ($768\leq d\leq 12,288$) and compared with the layer normalization method based on the fast inverse square root (FISR) algorithm~\cite{fast_inverse_square_root}. Since FISR is designed for FP formats with an 8b exponent, we limit our comparison to FP32 and BFloat16. The results are shown in Table~\ref{tab:error_comparison_fisr}. In FP32, IterL2Norm outperforms the FISR-based method in terms of average precision in six out of nine cases while, in BFloat16, it does so in five out of nine cases.

\begin{figure}[tb]
\centerline{\includegraphics[width=0.95\linewidth]{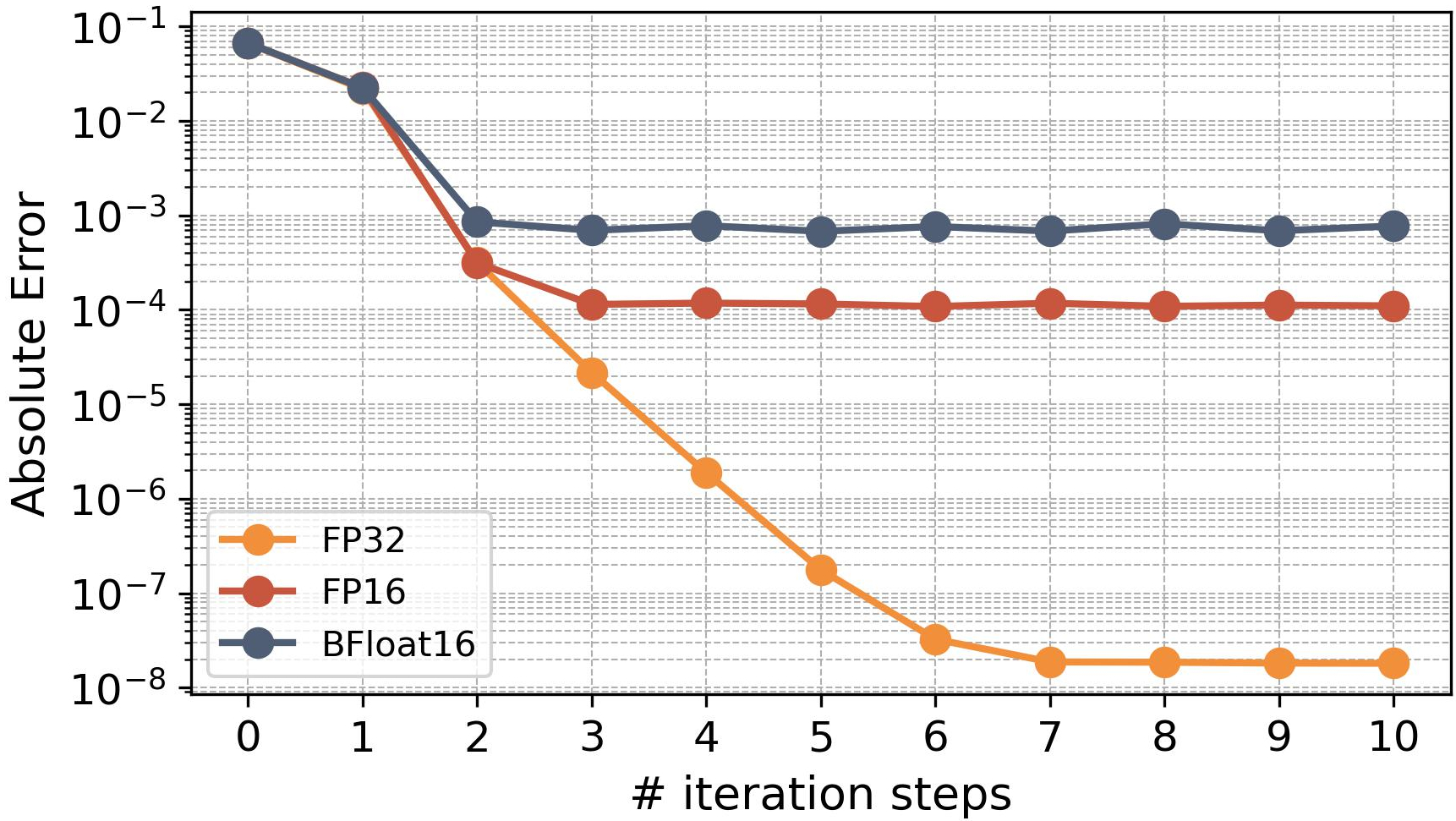}}
\caption{Average absolute errors of IterL2Norm in FP32, FP16, and BFloat16 with the number of iteration steps.}\label{fig:convergence_rate}

\end{figure}

\begin{table}[tb]
\caption{\label{tab:error_comparison_fisr} Precision comparison between IterL2Norm and FISR}
\begin{center}
\begin{adjustbox}{width=0.47\textwidth}
\begin{tabular}{ccccc}
\hline
\multirow{3}{*}{\textbf{\begin{tabular}[c]{@{}c@{}}Input\\ length\end{tabular}}} & \multicolumn{2}{c}{\textbf{FP32}} & \multicolumn{2}{c}{\textbf{BFloat16}} \\ \cline{2-5} 
 & \multicolumn{2}{c}{\textbf{Avg/Max Err$(\times10^{-4})$}} & \multicolumn{2}{c}{\textbf{Avg/Max Err$(\times10^{-3})$}} \\
 & \textbf{IterL2Norm} & \textbf{FISR} & \textbf{IterL2Norm} & \textbf{FISR} \\ \hline
768   & \textbf{0.132}/\textbf{29.35} & 4.124/101.6 & \textbf{2.195}/\textbf{125.0} & 2.294/125.0 \\
1024  & \textbf{1.987}/91.76 & 3.104/61.21 & 2.243/\textbf{125.0} & 2.235/125.0 \\
2048  & 61.76/3699.0 & 1.544/37.69 & 7.423/312.5 & 2.142/125.0 \\
2560  & \textbf{0.030}/\textbf{0.658} & 1.232/25.67 & \textbf{2.069}/\textbf{125.0} & 2.137/125.0 \\
4096  & 1.516/94.21 & 0.767/16.90 & \textbf{2.129}/\textbf{125.0} & 2.154/125.0 \\
5120  & \textbf{0.032}/\textbf{0.782} & 0.613/14.97 & \textbf{2.008}/\textbf{125.0} & 2.124/125.0 \\
7168  & 20.61/467.0 & 0.435/8.831 & 2.456/187.5 & 2.109/125.0 \\
9216  & \textbf{0.203}/14.98 & 0.337/8.736 & 2.160/\textbf{125.0} & 2.129/125.0 \\
12288 & \textbf{0.015}/\textbf{1.831} & 0.251/5.846 & \textbf{2.070}/\textbf{125.0} & 2.185/125.0 \\ \hline
\end{tabular}
\end{adjustbox}
\end{center}
\end{table}

To identify the covergence rate, we measured the average absolute error by varying the numbers of iteration steps for $d=1024$ for FP32, FP16, and BFloat16. Fig.~\ref{fig:convergence_rate} plots the evaluation results, where each data point was acquired from 1,000 trials. IterL2Norm in FP16 and BFloat16 ensures fast convergence within five iteration steps while that in FP32 needs a few additional iteration steps until convergence. Nevertheless, the error after five steps is close to the steady state error and far below the steady state errors for FP16 and BFloat16. Note that, in all formats, these errors after five steps may be sufficiently low to avoid LLM-level performance degradation on some text generation tasks as addressed in Sec.~\ref{sec:llm_eval}.     

\subsection{Operational latency}
\begin{figure}[tb]
\centerline{\includegraphics[width=0.95\linewidth]{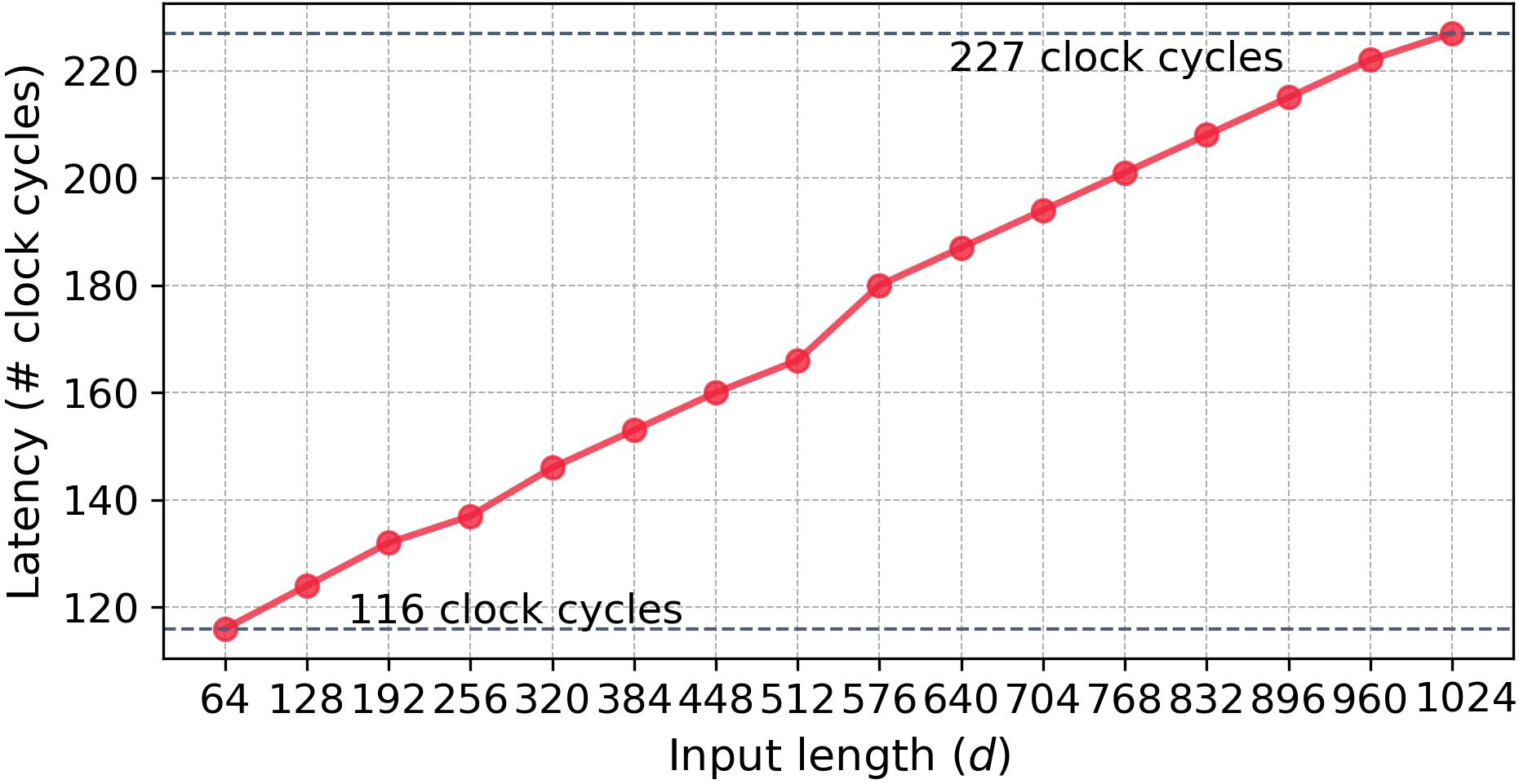}}
\caption{Measured latency of IterL2Norm (five iteration steps) with input length $d$.}
\label{fig:macro_latency}
\end{figure}

We evaluated the IterL2Norm macro latency by varying input length $d$. The results are plotted in Fig.~\ref{fig:macro_latency}. Note that the latency does not rely on the data format. The latency scales with the number of chunks $\lceil d/\left(n_bw_b\right)\rceil$ of the input length $d$. This is because the major steps addressed in Sec.~\ref{sec:macro_design} work on a chunk of $n_bw_b(=64)$ input elements at a time, so that the latency scales with $\lceil d/\left(n_bw_b\right)\rceil$. They include the mean calculation, mean-shift operation, inner product of $\boldsymbol{y}$ with itself, and scale and shift operations.

\subsection{IterL2Norm macro in 32/28nm CMOS}
\begin{table}[tb]
\caption{\label{tab:macro_synthesis}Synthesis results for the IterL2Norm macros in FP32, FP16, and BFloat16}
\begin{center}
\begin{adjustbox}{width=0.47\textwidth}
\begin{tabular}{ccccc}
\hline
\textbf{Format} & \textbf{Memory size} & \textbf{\# cells} & \textbf{Area} & \textbf{Power} \\ \hline
FP32 & 96.5 kib & 269.3k & 2.4 $(1.7)^{\dag}$ $\textrm{mm}^2$ & 22.9 mW \\
FP16 & 48.3 kib & 100.1k & 1.1 $(0.8)^{\dag}$ $\textrm{mm}^2$ & 8.4 mW \\
BFloat16 & 48.3 kib & 87.0k & 1.0 $(0.8)^{\dag}$ $\textrm{mm}^2$ & 7.3 mW \\ \hline
\end{tabular}
\end{adjustbox}
\begin{tablenotes}\small
\item[]{$\dag$: Area without the Add and Mul blocks.}\\
\end{tablenotes}
\end{center}
\end{table}

\begin{figure}[tb]
\centerline{\includegraphics[width=0.95\linewidth]{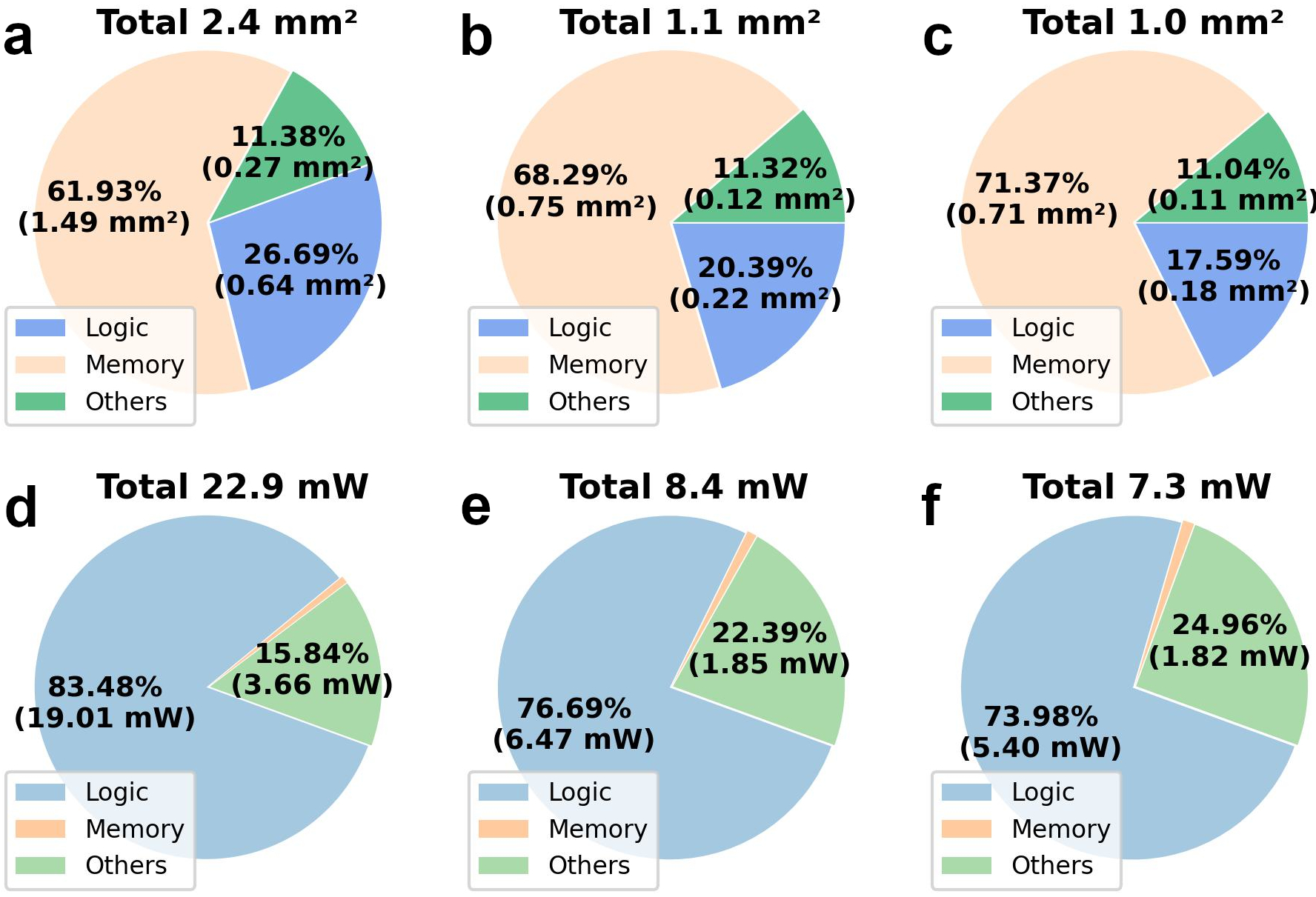}}
\caption{Area breakdowns for the IterL2Norm macro for (\textbf{a}) FP32, (\textbf{b}) FP16, and (\textbf{c}) BFloat16. Power breakdowns for (\textbf{d}) FP32, (\textbf{e}) FP16, and (\textbf{f}) BFloat16.}\label{fig:breakdown}
\end{figure}

\begin{table*}[tb!]
\caption{\label{tab:comparison}Comparison between the IterL2Norm macro and previous implementations of layer normalization}
\begin{center}
\begin{adjustbox}{width=0.91\textwidth}
\begin{tabular}{cccccccc}
\hline
 & \textbf{Implementation} & \textbf{Method} & \textbf{Operations} & \textbf{Data format} & \textbf{Area} & \textbf{Power} & \textbf{Clock frequency} \\ \hline
\cite{marchisio2023swifttron} & 65nm CMOS & approximate SQRT & \begin{tabular}[c]{@{}c@{}} addition, division, bit shift\end{tabular} & INT32 & 68.3 $\textrm{mm}^2$ & 2.0 W & 143 MHz \\ \hline
\cite{yu2022nn} & 7nm CMOS & approximate 1/SQRT & \begin{tabular}[c]{@{}c@{}}multiplication, addition\end{tabular} & \begin{tabular}[c]{@{}c@{}}INT32\\ FP32\\ FP16\end{tabular} & \begin{tabular}[c]{@{}c@{}}1008.9 ${\mu}m^2$\\ 1133.6 ${\mu}m^2$\\ 498.4 ${\mu}m^2$\end{tabular} & \begin{tabular}[c]{@{}c@{}}59.1 $\mu$W\\ 43.7 $\mu$W\\ 25.0 $\mu$W\end{tabular} & - \\ \hline
\cite{wu2024pim} & 28nm CMOS & FISR & \begin{tabular}[c]{@{}c@{}}multiplication, addition, bit shift\end{tabular} & BFloat16 & - & - & 1 GHz \\ \hline
\cite{wang2023sole} & 28nm CMOS & \begin{tabular}[c]{@{}c@{}}layer normalization\\ w/ dynamic compress\end{tabular} & \begin{tabular}[c]{@{}c@{}}multiplication, addition, bit shift\end{tabular} & INT8 & - & - & 1 GHz \\ \hline
\textbf{Ours} & \textbf{32/28nm CMOS} & \textbf{IterL2Norm} & \textbf{\begin{tabular}[c]{@{}c@{}}multiplication, addition\end{tabular}} & \textbf{\begin{tabular}[c]{@{}c@{}}FP32\\ FP16\\ BFloat16\end{tabular}} & \textbf{\begin{tabular}[c]{@{}c@{}}2.4 (1.7)$^{\dag}$ $\textrm{mm}^2$\\ 1.1 (0.8)$^{\dag}$ $\textrm{mm}^2$\\ 1.0 (0.8)$^{\dag}$ $\textrm{mm}^2$\end{tabular}} & \textbf{\begin{tabular}[c]{@{}c@{}}22.9 mW\\ 8.4 mW\\ 7.3 mW\end{tabular}} & \textbf{100 MHz} \\ \hline
\end{tabular}
\end{adjustbox}
\begin{tablenotes}\small
\item[]{$\dag$: Area without the Add and Mul blocks.}\\
\end{tablenotes}
\end{center}
\end{table*}

We finally synthesized the IterL2Norm macro for each of FP32/FP16/BFloat16 using the Synopsys SAED 32/28nm technology PDK with a supply voltage of 1.05V and a clock frequency $f_\text{clk}$ of 100MHz. We used the Design Compiler V-2023.12-SP5. Note that the configuration of Input buffer banks for all formats follows the generic architecture (use of eight Input buffer banks, each of which stores $16\times 8$ input elements) explained in Sec.~\ref{sec:macro_design}. The synthesis results are summarized in Table~\ref{tab:macro_synthesis}. As such, the macro for FP32 needs on-chip memory (96.5 kib) twice as large as those (48.3 kib) for FP16 and BFloat16. 
For FP32, each of the Input, $\gamma$, and $\beta$ buffers uses 32 kib to store 1024 elements maximally, and the partial sum buffer uses 0.5 kib to maximally store 16 partial sums. For FP16 and BFloat16, the memory usage is half that for FP32 such that the Input, $\gamma$, and $\beta$ buffers use 48 kib in total, and the partial sum buffer 0.25 kib. 

The number of standard cells used is primarily determined by the FP multipliers and adders. As such, among the three formats, the FP32 multiplier and adder require the most standard cells
due to their higher number of exponent and mantissa bits. The BFloat16 multiplier and adder require a fewer standard cells than FP16 because of their lower number of mantissa bits that are subject to multiplication and addition.

As shown in Table~\ref{tab:macro_synthesis}, the macro areas for FP32, FP16, and BFloat16 are 2.4, 1.1, and 1.0 $\textrm{mm}^2$, respectively. The area breakdown for each format is shown in Figs.\ref{fig:breakdown}\textbf{a}-\textbf{c}. For all formats, the memory (Input/$\gamma$/$\beta$ and partial sum buffers) occupies the largest area in the macro, which is followed by the logic area including FP multipliers and adders. Although we considered FP multipliers and adders dedicated solely to IterL2Norm, IterL2Norm can use them in the MatMul block co-integrated on the same die. Therefore, the actual area of the IterL2Norm macro likely excludes the multiplier and adder areas, which is also listed in Table~\ref{tab:macro_synthesis}.   
The operational power is also primarily determined by the FP multipliers and adders, resulting in the highest (lowest) power consumption for FP32 (BFloat16) as identified in Table~\ref{tab:macro_synthesis} and power breakdown for each format in Figs.\ref{fig:breakdown}\textbf{d}-\textbf{f}.

\subsection{LLM-level evaluation}\label{sec:llm_eval}

\begin{table}[tb]
\caption{\label{tab:llm_results} LLM-level evaluation of IterL2Norm using OPT-125M and OPT-350M on two text generation tasks}
\begin{center}
\small
\begin{adjustbox}{width=0.47\textwidth}
\begin{tabular}{cccccccc}
\hline
\multirow{2}{*}{\textbf{Task}} & \multirow{2}{*}{\textbf{Format}} & \multicolumn{3}{c}{\textbf{OPT-125M}} & \multicolumn{3}{c}{\textbf{OPT-350M}} \\ \cline{3-8} 
 &  & \textbf{Baseline} & \textbf{\# steps} & \textbf{Perplexity} & \textbf{Baseline} & \textbf{\# steps} & \textbf{Perplexity} \\ \hline
\multirow{12}{*}{Wikitext-2} & \multirow{4}{*}{FP32} & \multirow{4}{*}{18.21} & 3 & 18.37 (+0.16) & \multirow{4}{*}{15.28} & 3 & 15.28 (+0.00) \\
 &  &  & 4 & 18.22 (+0.01) &  & 4 & 15.28 (+0.00) \\
 &  &  & 5 & 18.21 (+0.00) &  & 5 & 15.28 (+0.00) \\
 &  &  & 10 & 18.21 (+0.00) &  & 10 & 15.28 (+0.00) \\ \cline{2-8} 
 & \multirow{4}{*}{FP16} & \multirow{4}{*}{25.35} & 3 & 25.51 (+0.16) & \multirow{4}{*}{27.57} & 3 & 27.57 (+0.00) \\
 &  &  & 4 & 25.35 (+0.00) &  & 4 & 27.57 (+0.00) \\
 &  &  & 5 & 25.35 (+0.00) &  & 5 & 27.57 (+0.00) \\
 &  &  & 10 & 25.35 (+0.00) &  & 10 & 27.57 (+0.00) \\ \cline{2-8} 
 & \multirow{4}{*}{BFloat16} & \multirow{4}{*}{19.17} & 3 & 19.43 (+0.26) & \multirow{4}{*}{15.43} & 3 & 15.44 (+0.01) \\
 &  &  & 4 & 19.20 (+0.03) &  & 4 & 15.43 (+0.00) \\
 &  &  & 5 & 19.20 (+0.03) &  & 5 & 15.43 (+0.00) \\
 &  &  & 10 & 19.17 (+0.00) &  & 10 & 15.43 (+0.00) \\ \hline
\multirow{12}{*}{BST} & \multirow{4}{*}{FP32} & \multirow{4}{*}{17.30} & 3 & 17.36 (+0.06) & \multirow{4}{*}{15.41} & 3 & 15.41 (+0.00) \\
 &  &  & 4 & 17.31 (+0.01) &  & 4 & 15.41 (+0.00) \\
 &  &  & 5 & 17.30 (+0.00) &  & 5 & 15.41 (+0.00) \\
 &  &  & 10 & 17.30 (+0.00) &  & 10 & 15.41 (+0.00) \\ \cline{2-8} 
 & \multirow{4}{*}{FP16} & \multirow{4}{*}{19.61} & 3 & 19.67 (+0.16) & \multirow{4}{*}{21.94} & 3 & 21.95 (+0.01) \\
 &  &  & 4 & 19.61 (+0.00) &  & 4 & 21.94 (+0.00) \\
 &  &  & 5 & 19.61 (+0.00) &  & 5 & 21.94 (+0.00) \\
 &  &  & 10 & 19.61 (+0.00) &  & 10 & 21.94 (+0.00) \\ \cline{2-8} 
 & \multirow{4}{*}{BFloat16} & \multirow{4}{*}{17.83} & 3 & 17.91 (+0.08) & \multirow{4}{*}{15.49} & 3 & 15.49 (+0.00) \\
 &  &  & 4 & 17.84 (+0.01) &  & 4 & 15.49 (+0.00) \\
 &  &  & 5 & 17.84 (+0.01) &  & 5 & 15.49 (+0.00) \\
 &  &  & 10 & 17.83 (+0.00) &  & 10 & 15.49 (+0.00) \\ \hline
\end{tabular}
\end{adjustbox}
\end{center}
\end{table}

We evaluated IterL2Norm at the LLM-level by creating a PyTorch module for IterL2Norm. Note that the precision of IterL2Norm in PyTorch negligibly differs from that in FPGA. We considered the Open Pre-trained Transformer (OPT) models with 125M (OPT-125M) and 350M (OPT-350M) parameters~\cite{zhang2022opt} for text generation. The OPT models are decoder-only models consisting of stacks of 12 and 24 transformer blocks, each employing 12 and 16 attention heads with embedding sizes of 768 and 1024, respectively. We used two text generation datasets: WikiText-2~\cite{merity2016pointer} and Blended Skill Talk (BST)~\cite{smith2020can}. All layer normalization blocks in the pre-trained OPT-125M and OPT-350M models for each dataset were replaced with IterL2Norm, and we measured the change in perplexity scores for text generation as the LLM-level error of IterL2Norm. The IterL2Norm module takes the number of iteration steps $n_\text{iter}$ as a parameter. 
The perplexity scores for different iteration steps in FP32, FP16, and BFloat16 are listed in Table~\ref{tab:llm_results}. Compared to the baseline perplexity, the perplexity scores for both WikiText-2 and BST marginally increased after the third iteration step.

\section{Related work and comparison}\label{sec:comp}
Our IterL2Nom macros designed for FP32, FP16, and BFloat16 are compared with previous implementations of layer normalization in Table~\ref{tab:comparison}.  
The method in \cite{marchisio2023swifttron} realizes the layer normalization of INT32 vectors using integer-only arithmetic. To this end, they adopted an iterative algorithm~\cite{crandall2005prime} that approximates square root values. Thus, this method requires additional division operations to normalize the input integer vector. When implemented in 65nm CMOS, the area and power overheads of the circuit are 68.3 mm$^{2}$ and 2.0 W, respectively.

The method in~\cite{yu2022nn} avoids costly division operations for normalization by using a lookup-table (LUT)-based approximation of the inverse square root function. The inverse square root function is approximated using a piecewise linear method, and the function values for multiple inputs are stored in an LUT. For a given input, its square root value is calculated by interpolating between two neighboring function values. Wu et al. implemented FISR~\cite{fast_inverse_square_root} in BFloat16 in 28nm CMOS technology for on-chip layer normalization~\cite{wu2024pim}. Unfortunately, the detailed implementation and performance data are unavailable. 

The method in~\cite{wang2023sole} uses the low-precision computation of the mean and standard deviation using dynamic compression and power-of-two factor quantization methods. The mean and standard deviation values are computed using 4-bit integer arithmetic. Additional LUTs are used to store inverse square root values. However, similar to~\cite{wu2024pim}, the implementation and performance data are missing.

\section{Conclusion}\label{sec:conclusion}
We introduced IterL2Norm, an efficient method for iteratively L2-normalizing input vectors without costly division or square root operations. Grounded in solid theory, IterL2Norm is applicable to general FP data and ensures high precision, outperforming FISR in six out of nine cases for FP32 and five out of nine for BFloat16 across the embedding lengths used in the OPT models. It also converges quickly, reaching its fixed point within five iterations. Implemented in 32/28nm CMOS technology, the IterL2Norm macro processes $d$-dimensional input vectors, where $64 \leq d \leq 1024$, with a latency of 116-227 clock cycles.


\bibliography{Bib.bib}

\end{document}